\documentclass[12pt]{article}
\usepackage{amssymb,amsmath,amsthm}
\usepackage{graphicx, color}

\usepackage{geometry}
\usepackage{setspace}
\usepackage{verbatim}
\usepackage{enumerate}
 
\newtheorem{theorem}{Theorem}[section]
\newtheorem{lemma}[theorem]{Lemma}

\newtheorem*{theorem*}{Theorem}
\newtheorem*{lemma*}{Lemma}

\newtheorem*{remark}{Remark}

\usepackage[pdftex,pagebackref=true,colorlinks]{hyperref}

\newcommand{\F}{\mathbb F}

\newcommand{\hh}{\mathcal{H}}
\newcommand{\Ex}{\mathop \mathbb{E}}
\newcommand{\Var}{\text{Var}}

\newcommand{\eps}{\epsilon}

\newcommand{\X}{\mathcal{X}}
\newcommand{\Y}{\mathcal{Y}}
\newcommand{\Z}{\mathcal{Z}}

\newcommand{\E}{\mathbb{E}}
\newcommand{{\precomp}}{{selection scheme}}

\renewcommand\H{{\cal H}}

\begin{document}

\title{On statistical learning via the lens of compression}

\author{
Ofir David\thanks{Department of Mathematics, Technion-IIT, Israel. {\tt eofirdavid@gmail.com}.}
\and Shay Moran\thanks{Departments of Computer Science, Technion-IIT, Israel, Microsoft Research Hertzelia, and Max Planck Institute for Informatics, Saarbr\"{u}cken, Germany. {\tt  shaymrn@cs.technion.ac.il.}}
\and Amir Yehudayoff\thanks{Department of Mathematics, Technion-IIT, Israel. {\tt amir.yehudayoff@gmail.com.} Research is supported by ISF and BSF.}
}

\date{}

\maketitle

\begin{abstract}
This work continues the study of the relationship
between sample compression schemes and statistical learning,
which has been mostly investigated within the framework of binary classification.
{The central theme of this work is establishing equivalences
between learnability and compressibility, 
and utilizing these equivalences in the study of statistical learning theory.}
%

We begin with the setting of {multiclass categorization} (zero/one loss).
We prove that in this case learnability is equivalent
to compression of logarithmic sample size, and {that}
uniform convergence implies compression of constant size.

{We then consider Vapnik's general learning setting:
we show that in order to extend the compressibility-learnability equivalence to this case, 
it is necessary to consider an approximate variant of compression.}

{Finally, we provide some applications of the 
compressibility-learnability equivalences:}
\begin{itemize}
\item Agnostic-case learnability and realizable-case learnability are equivalent in multiclass categorization problems
({in terms of sample complexity}).
\item  This equivalence {between agnostic-case learnability and realizable-case learnability does not hold} for general learning problems: There exists a learning problem whose loss function takes just three values, under which agnostic-case and realizable-case learnability are not equivalent.
\item
Uniform convergence implies compression {of constant size} in multiclass categorization problems. 
Part of the argument includes an analysis of the uniform convergence
rate in terms of the graph dimension, in which we improve upon previous bounds.
\item 
A dichotomy for sample compression in multiclass categorization problems:
{If} a non-trivial compression exists 
then a compression {of} logarithmic size exists.
\item A compactness theorem for multiclass categorization problems.
\end{itemize}
\end{abstract}

\thispagestyle{empty}

\newpage

\thispagestyle{empty}

\begin{center}
\includegraphics[width=.8\textwidth,height=.8\textheight]{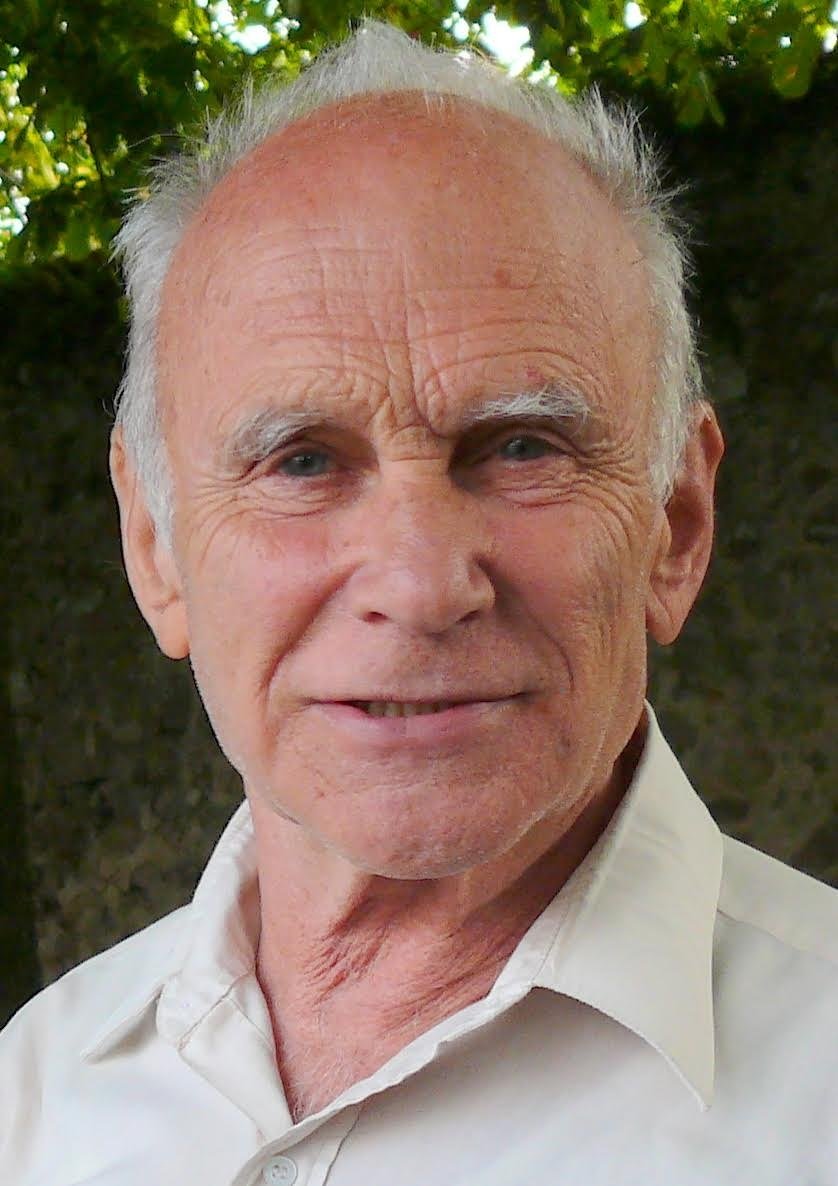}
\end{center}

\

\begin{center}
The second author dedicates this work to the memory of Gadi Moran (16.05.38 | 01.01.16)
for exposing him to the beauty of mathematics, and for insightful discussions about Ramsey Theory in the context of this work.
\end{center}
\newpage

\setcounter{page}{1}

\section{Introduction}

This work studies statistical learning theory
using the point of view of compression.
The main theme in this work {is} 
{{\it establishing
equivalences between learnability and compressibility,
and {making an effective use of these equivalences} to study statistical learning theory.}

{
In a nutshell, the usefulness of these equivalences stems from that
compressibility
is a combinatorial notion, while learnability is a statistical notion.
These equivalences, therefore, 
translate statistical statements to combinatorial ones
and vice versa.
This translation helps to reveal properties
that are otherwise difficult to find,
and highlights useful guidelines for designing learning algorithms.}
%


We first consider the setting of {\it multiclass categorization}, which is used to
model supervised learning problems using the zero/one loss function, 
and then move to {\it Vapnik's general learning setting} \cite{DBLP:books/daglib/0097035},
which models many supervised and unsupervised learning problems.}
Readers that are not familiar with the relevant definitions
are referred to Section~\ref{sec:prel}.

\paragraph{Multiclass categorization (Section~\ref{sec:zoLoss}).}
This is the setting in which sample compression schemes were defined by Littlestone and Warmuth~\cite{littleWarm},
as an abstraction of a common property of many learning algorithms.
For more background on sample compression schemes,
see e.g.~\cite{littleWarm,DBLP:conf/colt/Floyd89,DBLP:journals/ml/FloydW95,Shalev-Shwartz}.

We use an agnostic version of sample compression schemes,
and show that learnability 
is equivalent to some sort of compression.
{More formally, that any learning algorithm can be transformed
to a compression algorithm, compressing a sample of size $m$ to a sub-sample of size
roughly $\log(m)$},
and that such a compression algorithm implies learning.
This statement is based on arguments that appear in~\cite{littleWarm,DBLP:journals/iandc/Freund95,schapire2012boosting}.
We conclude this part by describing some applications:

(i)~Equivalence between PAC and agnostic PAC learning
from a statistical perspective {(i.e.\
in terms of sample complexity)}.
For binary-labelled classes,
this equivalence follows from basic arguments in Vapnik-Chervonenkis (VC) theory,
but these arguments do not seem to extend when the number of labels is large.

(ii)~A dichotomy for sample compression |
if a non-trivial compression exists (e.g.\ compressing a sample of size $m$ to 
a sub-sample of size $m^{0.99}$),
then a compression to logarithmic size exists (i.e.\ to 
a sub-sample of size roughly $\log m$). 
{This dichotomy is analogous to the known dichotomy
concerning the growth function of binary-labelled classes: 
the growth function is either polynomial (when the VC dimension is finite), or exponential (when the VC dimension is infinite).}

(iii)~Compression to constant size versus uniform convergence |
every class with
the uniform convergence property has a compression of constant size.
The proof has two parts.
The first part, which is based on
arguments from~\cite{DBLP:journals/eccc/MoranY15},
shows that finite graph dimension (a generalization of VC dimension for multiclass categorization~\cite{Natarajan89}) implies compression of constant size.
The second part, which uses ideas from
~\cite{BenDavid95,zbMATH03391742,DBLP:journals/corr/DanielySBS13}, 
shows that the uniform convergence rate is captured by the graph dimension.
In this part we improve upon the previously known bounds.

(iv)~Compactness for learning |
if finite sub-classes of a given class are learnable,
then the class is learnable as well.
Again, for binary-labelled classes,
such compactness easily follows from known
properties of VC dimension.
For general multi-labeled classes we
derive this statement using a corresponding compactness
property for sample compression schemes,
based on the work by~\cite{DBLP:journals/dam/Ben-DavidL98}.

\paragraph{General learning setting (Section~\ref{sec:GenLoss}).}
We continue with investigating general loss functions.
This part begins with a simple example
in the context of linear regression, showing that
for general loss functions,
learning is not equivalent to compression.
We then consider an approximate variant of compression schemes,
which was used by~\cite{DBLP:journals/ml/GraepelHS05,DBLP:journals/corr/GottliebK15} 
in the context of classification,
and observe that learnability is equivalent to possessing an
approximate compression scheme,
whose size is roughly the statistical sample complexity.
This is in contrast to
(standard) sample compression schemes,
for which the existence of such an equivalence
(under the zero/one loss) 
is a long standing open problem, even in the case of binary classification~\cite{DBLP:conf/colt/Warmuth03}.
We conclude the paper by showing that |
unlike for zero/one loss functions |
for general loss functions,
PAC learnability and agnostic PAC learnability are {\em not} equivalent.
In fact, this is derived for a loss function that takes just three values.
The proof of this non-equivalence uses Ramsey theory
for hypergraphs.
{The combinatorial nature of compression schemes
allows to clearly identify the place where Ramsey theory
is helpful.
More generally, the study of statistical learning theory via the lens
of compression may shed light on additional useful connections with different fields of mathematics.}

\paragraph{Selection schemes.}
We begin our investigation by breaking the definition of sample
compression schemes into two parts.
The first part (which may seem useless at first sight)
{is about {\it selection schemes}. These are learning algorithms whose output hypothesis depends on
a selected small sub-sample of the input sample.
The second part of the definition is the sample-consistency guarantee;
so, sample compression schemes are selection schemes
whose output hypothesis is consistent with the input sample.}
We then show that selection schemes of small size do not overfit
in that their empirical risk is close to their true risk.
{Roughly speaking, this shows that for selection
schemes there are no surprises:
``what you see is what you get''}.

\section{Preliminaries}
\label{sec:prel}

The definitions we use are based on the textbook by~\cite{Shalev-Shwartz}.

\subsubsection*{Learnability and uniform convergence}
{
A learning problem is specified by a set $\hh$ of hypotheses,
a domain $\Z$ of examples, and a loss function $\ell:\hh\times\Z\rightarrow\mathbb{R}^+$.
To ease the presentation, we shall only discuss loss functions that are bounded from above by~$1$, although the results presented here can be extended to more general loss functions.
A sample $S$ is a finite sequence $S=(z_1,\ldots,z_m) \in \Z^m$.
A {\em learning algorithm} is a mapping that gets as an input a sample
and outputs an hypothesis $h$.

In the context of supervised learning, hypotheses are functions from a domain $\mathcal{X}$ to a label set $\mathcal{Y}$, and the examples domain is the cartesian product $\Z:=\X\times\Y$.
In this context, the loss $\ell(h,(x,y))$ depends only on $h(x)$ and $y$, 
and therefore in this case we model the loss as a function $\ell:\mathcal{Y}\times\mathcal{Y}\rightarrow\mathbb{R}^+$.
}
Examples of loss functions include:
\begin{description}
  \item[Multiclass categorization] The zero/one loss function
   $\ell(y_1,y_2)=\begin{cases} 0 &y_1=y_2 , \\ 1 &y_1\neq y_2 .\end{cases}$
  \item[Regression] The squared loss function over $\Y=[0,1]$
  is $\ell(y_1,y_2) = (y_1-y_2)^2$.
\end{description}

  \medskip

Given a distribution $\mathcal{D}$ on $\Z$,
the {\em risk} of an hypothesis $h:\mathcal{X}\rightarrow\mathcal{Y}$ is its expected loss:
$$L_{\mathcal{D}}(h) = \E_{z\sim\mathcal{D}}\left[\ell(h,z)\right].$$
Given a sample $S=(z_1,\ldots,z_m)$, the {\em empirical risk} of an hypothesis $h$
is
$$L_S(h) = \frac{1}{m}\sum_{i=1}^{m}{\ell(h,z)}.$$

An {\em hypothesis class} $\hh$ is a set of hypotheses.
A distribution $\mathcal{D}$ is {\em realizable} by $\hh$ if there exists $h\in\hh$ such that $L_\mathcal{D}(h)=0$.
A sample $S$ is {\em realizable} by $\hh$ if there exists $h\in\hh$ such that $L_S(h)=0$.

A hypothesis class $\hh$ has {\em the uniform convergence property}\footnote{We omit the dependence on the loss function $\ell$ from this and similar definitions,
since $\ell$ is clear from the context.} if
there exists a rate function $d:(0,1)^2\rightarrow\mathbb{N}$ such that
for every $\eps,\delta > 0$ and distribution $\mathcal{D}$ over $\Z$,
if $S$ is a sample of $m\geq d(\eps,\delta)$ i.i.d.\ pairs generated
by $\mathcal{D}$, then with probability at least $1-\delta$ we have
$$\forall h\in\hh \ ~ |L_{\mathcal{D}}(h)- L_S(h)| \leq  \eps.$$

The class $\hh$ is {\em agnostic PAC learnable} if there exists a learner $A$ and a rate function $d:(0,1)^2\rightarrow\mathbb{N}$ such that for every $\eps,\delta > 0$ and distribution $\mathcal{D}$ over $\Z$,
if $S$ is a sample of $m\geq d(\eps,\delta)$ i.i.d.\ pairs generated
by $\mathcal{D}$, then with probability at least $1-\delta$
we have
\begin{equation}\label{eq:agn}
L_{\mathcal{D}}(A(S)) \leq \inf_{h\in\hh}L_{\mathcal{D}}(h) + \eps.
\end{equation}
The class
$\hh$ is {\em PAC learnable} if condition~\eqref{eq:agn} holds for every realizable distribution
$\mathcal{D}$.
The parameter
$\eps$ is referred to as the {\em error} parameter and $\delta$ as the {\em confidence} parameter.

Note that the uniform convergence property implies agnostic PAC learnability with the same rate
via any learning algorithm which outputs $h\in\hh$
that minimizes the empirical risk, and that agnostic PAC learnability implies
PAC learnability with the same rate.

\subsubsection*{Selection and compression schemes}
The variants of sample compression schemes that are discussed in this paper,
are based on the following object, which we term {\em {\precomp}}. 
{We stress here that unlike sample compression schemes, selection schemes are not
associated with any hypothesis class.}

A {{\precomp}} is a pair $(\kappa,\rho)$ of maps for which the following holds:
\begin{itemize}
\item $\kappa$ is called the selection map. It gets as an input a sample $S$ and outputs
a pair $(S',b)$ where $S'$ is a sub-sample\footnote{That is,
if $S = (z_1,\ldots,z_m)$ then
$S'$ is of the form $(z_{i_1},\ldots,z_{i_\ell})$
for $1 \leq {i_1} < \ldots < {i_\ell} \leq m$.} of $S$  and $b$ is a finite binary string, which we think of as side information.
\item $\rho$ is called the reconstruction map. It gets as an input a pair $(S',b)$ of the same type as the output of $\kappa$
and outputs an hypothesis $h$.
\end{itemize}

The size of $(\kappa,\rho)$ on a given input sample $S$ is defined to be $|S'|+|b|$ where $\kappa(S)=(S',b)$.
For an input size $m$, we denote by $k(m)$ the maximum size of the {{\precomp}} on all inputs $S$ of size at most $m$.
{The function $k(m)$ is called the {\em size} of the selection scheme.}
If $k(m)$ is uniformly bounded by a constant, which does not depend on $m$, then we say that the {\precomp} has a constant size; otherwise, we say that it has a variable size.


The definition of {\precomp} is very similar
to that of sample compression schemes. 
The difference is that sample compression schemes are defined with respect to a fixed hypothesis class
with respect to which they
are required to have ``correct'' reconstructions whereas
selection schemes do not provide any correctness guarantee.
The distinction between the `selection' part and
the `correctness' part is helpful for our presentation,
and also provides some more insight into these notions.

A {\precomp} $(\kappa,\rho)$ is a {\em sample compression scheme} for $\hh$ if
for every sample $S$ that is realizable by $\hh$,
$$L_{S}\left(\rho\left(\kappa\left(S\right)\right)\right)=0.$$
A {\precomp} $(\kappa,\rho)$ is an {\em agnostic sample compression scheme} for $\hh$ if
for every sample $S$,
$$
L_{S}\left(\rho\left(\kappa\left(S\right)\right)\right)\leq \inf_{h\in\hh}L_{S}(h).
$$

In the following sections, we will see different manifestations of the statement ``compression $\Rightarrow$ learning''.
An essential part of these statements boils down to a basic property of selection schemes, that as long as $k(m)$ is sufficiently smaller than $m$, a selection scheme based learner does not overfit its training data in the sense that its risk is roughly its empirical risk.
For completeness we provide a proof of it in Section~\ref{s:appt11}.

\newcommand{\D}{{\cal D}}

\begin{theorem}[{\cite[Theorem 30.2]{Shalev-Shwartz}}]
\label{t11}
Let $(\kappa,\rho)$ be a {\precomp} of size $k=k(m)$, and let $A(S) = \rho\left(\kappa\left(S\right)\right)$.
Then, for every distribution $\D$ on $\Z$, integer $m$ such that $k\leq m/2$, and $\delta>0$, we have
$$\Pr_{S\sim \D^m}\left[\lvert L_\D\left(A\left(S\right)\right) - L_S\left(A\left(S\right)\right) \rvert\geq
\sqrt{\eps\cdot L_S\left(A\left(S\right)\right)} + \eps \right]\leq \delta,$$
where
$$\eps = 50\frac{k\log\left(m/k\right)+\log(1/\delta)}{m}.$$
\end{theorem}

\section{Multiclass categorization}
\label{sec:zoLoss}

In this section we consider the zero/one loss function, which models categorization problems.
We study the relationships between uniform convergence, learnability, and sample compression schemes under this loss.
Subsection~\ref{s21} establishes equivalence between learnability and compressibility of a sublinear size. 
In Subsection~\ref{s22} we use this equivalence to study the relationships between the properties of uniform convergence, PAC, and agnostic PAC learnability. In Subsection~\ref{s221} we show that agnostic  PAC learnability is equivalent to PAC learnability,
in Subsection~\ref{s222} we discuss the role sample compression schemes have in the context of boosting.
In Subsection~\ref{s223} we observe a dichotomy concerning the size of sample compression schemes, and use it to establish a compactness property of learnability.
Finally, in Subsection~\ref{s224} we study an extension of the Littlestone-Floyd-Warmuth conjecture concerning an equivalence between learnability and sample compression schemes of fixed size.

\subsection{Learning is equivalent to sublinear compressing}\label{s21}

The following theorem shows that
if $\hh$ has a sample compression scheme of size $k=o(m)$, then it is learnable.

\begin{theorem}[Compressing implies learning~\cite{littleWarm}]
\label{t21}
Let $(\kappa,\rho)$ be a {\precomp} of size $k$, let $\hh$ be an hypothesis class,
and let  $\D$ be a distribution on $\Z$.
\begin{enumerate}
\item If $(\kappa,\rho)$ is a sample compression scheme for $\hh$, and $m$ is such that $k(m)\leq m/2$, then
$$\Pr_{S\sim \D^m}\left(L_\D\left(\rho\left(\kappa\left(S\right)\right)\right)> 50\frac{k\log\frac{m}{k} + k +\log\frac{1}{\delta}}{m} \right)<\delta.$$
\item If $(\kappa,\rho)$ is an agnostic sample compression scheme for $\hh$, and $m$ is such that $k(m)\leq m/2$, then
$$\Pr_{S\sim \D^m}\left(L_\D\left(\rho\left(\kappa\left(S\right)\right)\right)> \inf_{h\in \hh}L_\D(h) + 100\sqrt{\frac{k\log \frac{m}{k} + k +\log\frac{1}{\delta}}{m}}\right)<\delta.$$
\end{enumerate}
\end{theorem}

\begin{proof}
The first item follows immediately from Theorem~\ref{t11} by plugging $L_S(\rho(\kappa(S)))=0$.

For the second item we need the following lemma that we prove in Section~\ref{app:l22}.
\begin{lemma}\label{l22}
For every distribution $\D$ on $\X\times\Y$, $m\in\mathbb{N}$, and $\delta>0$:
$$\Pr_{S\sim D^m}\left(L_S(\rho(\kappa(S))) \geq \inf_{h\in \hh}L_\D(h) + \eps_1(m,\delta)\right)\leq \delta,$$
where $\eps_1(m,\delta)=\sqrt{\frac{\log\frac{1}{\delta}}{m}}$.
\end{lemma}

Plugging $\delta\leftarrow\delta/2$ in Lemma~\ref{l22}, and Theorem~\ref{t11}
yields the second item.

\end{proof}

The following theorem shows that learning
implies compression.

\begin{theorem}[Learning implies compressing]
\label{t22}
Let $\hh$ be an hypothesis class.
\begin{enumerate}
\item If $\hh$ is agnostic PAC learnable with learning rate $d(\epsilon,\delta)$, then it is PAC learnable with the same learning rate.
\item If $\hh$ is PAC learnable with learning rate $d(\eps,\delta)$, then it has a sample compression scheme of size
$k(m)=O(d_0\log(m) \log\log (m) + d_0\log(m)\log(d_0))$,
where $d_0=d(1/3,1/3)$.
\item If $\hh$ has a sample compression scheme of size $k(m)$, then it has an agnostic sample compression scheme of the same size.
\end{enumerate}
\end{theorem}
\begin{proof}
The first item follows directly from the definition of agnostic and PAC learnability.
The second item can be proven by boosting (see, e.g.~\cite{schapire2012boosting}).
For completeness, in Section~\ref{app:t23} we present a proof of this part which is based on von Neumann's minimax Theorem~\cite{Neumann1928}.
The last item follows from the observation that under the zero/one loss function,
any sample compression scheme $(\kappa,\rho)$ can be transformed to an agnostic sample compression scheme without increasing the size.
Indeed, suppose $(\kappa,\rho)$ is a sample compression scheme for $\hh$. Now, given an arbitrary sample (not necessarily realizable) $S$, pick some $h^* \in \mathcal{H}$ that minimizes $L_S(h)$;
the minimum is attained since the loss function is zero/one.
Denote by $\tilde S$ the sub-sample of $S$ on which $h^*$ agrees with $S$, so that by definition $\tilde S$ is realizable. Therefore, since $(\kappa,\rho)$ is a sample compression scheme for $\hh$, it follows that $L_{\tilde S}(\rho(\kappa(\tilde S)))=0$. In other words, applying the compression scheme on $\tilde S$ yields an hypothesis $\tilde h=\rho(\kappa({\tilde S}))$ which agrees with $h^*$ on $\tilde S$. Since the loss function is zero/one loss function, $\tilde h$ cannot be worse than $h^*$ on the part of $S$ that is outside of $\tilde S$.
Hence, $L_S(\tilde h) \leq \min_{h\in \H}L_S(h)$ as required.
\end{proof}

\begin{remark}
{The third part in Theorem~\ref{t22} does not hold when the loss function is general. In Section~\ref{sec:GenLoss} we show that even if the loss function takes three possible values, then there are instances where a class has a sample compression scheme but not an agnostic sample compression scheme.}
\end{remark}

\subsection{Applications}\label{s22}

\subsubsection{Agnostic and PAC learnability are equivalent}\label{s221}\label{sec:PACeqAGN}
Theorems~\ref{t21} and \ref{t22} imply that if $\hh$
is PAC learnable, then it is agnostic PAC learnable.
Indeed, a summary of the implications between learnability and compression given by Theorems~\ref{t21} and \ref{t22} gives:

\begin{itemize}
\item An \underline{agnostic learner} with rate $d\left(\epsilon,\delta\right)$
implies a \underline{PAC learner} with rate $d\left(\epsilon,\delta\right)$.
\item A \underline{PAC learner} with rate $d\left(\epsilon,\delta\right)$ implies a
\underline{sample compression scheme} of size 
$k\left(m\right)=O\left(d_{0}\cdot\log\left(m\right)\log\left(d_{0}\cdot\log\left(m\right)\right)\right)$
where $d_{0}=d(1/{3},{1}/{3})$.
\item A \underline{sample compression scheme} of size $k\left(m\right)$ implies an \underline{agnostic
sample compression scheme} of size $k\left(m\right)$.
\item An \underline{agnostic sample compression scheme} of size $k\left(m\right)$ implies an \underline{agnostic learner} with error  $\epsilon\left(d,\delta\right)=100\sqrt{\frac{k\left(d\right)\log\frac{d}{k\left(d\right)}+k(d)+\log\frac{1}{\delta}}{d}}$.
\end{itemize}


Thus, for multiclass categorization problems, agnostic PAC learnability and PAC learnability are equivalent.
When the size of the label set $\Y$ is $O(1)$, this equivalence follows from previous works that studied extensions of the VC dimension to multiclass categorization problems~\cite{zbMATH03391742,zbMATH04143473,Natarajan89,BenDavid95}. These works show that PAC learnability and agnostic PAC learnability are equivalent to the uniform convergence property, and therefore any ERM algorithm learns the class. Recently,~\cite{DBLP:journals/corr/DanielySBS13}  separated PAC learnability and uniform convergence for large label sets by exhibiting PAC learnable hypothesis classes that do not satisfy the uniform convergence property.
In contrast, this shows that the equivalence between
PAC and agnostic PAC learnability remains {valid} even when $\Y$ is large.


\subsubsection{Boosting}\label{s222}
Boosting refers to the task of {efficiently} transforming a weak learning algorithm (say) with confidence $2/3$,
error $1/3$, and rate $d$ examples to a strong learner with confidence $1-\delta$ and error $\eps$,
for some prescribed $\eps,\delta>0$.
{Beside the computational aspect},
this task {also} manifests a statistical aspect.
The statistical aspect concerns the minimum number of examples $d(\eps,\delta)$ that are required
in order to achieve the prescribed confidence and error.

The computational aspect was studied extensively and merited the celebrated Adaboost algorithm
(see the book by~\cite{schapire2012boosting} and references therein).
The statistical aspect, at least when the label set is binary, follows from
basic results in VC theory that characterize the uniform
convergence rate in terms of the VC dimension and establish
equivalence between learnability and uniform convergence.
However, when $\Y$ is infinite,
uniform convergence and learnability
cease to be equivalent, and therefore these arguments do not hold.

Freund and Schapire (see~\cite{schapire2012boosting} and references therein) showed that Adaboost is in fact a sample compression scheme and that this fact implies boosting.
Unlike the equivalence between learning and uniform convergence that breaks for large $\Y$,
the equivalence between compression and learning remains valid, and therefore the implication ``compression $\Rightarrow$ boosting'' extends to an arbitrary $\Y$. Indeed,
Theorems~\ref{t22} and~\ref{t21} imply a boosting of a weak learning algorithm to a strong learning algorithm
with sample complexity of the form $poly(1/\eps,\log(1/\delta))$,
where the ``$poly$'' notation hides dependency on the sample complexity of the weak learner:
{Theorem~\ref{t22} shows that a weak learner yields a sample compression scheme, and Theorem~\ref{t21}
shows that a sample compression scheme yields a strong learner.}

\subsubsection{A dichotomy and compactness}\label{s223}

Let $\hh$ be an hypothesis class. Assume e.g.\ that $\hh$ has a sample compression scheme of size $m/500$ for some
large $m$.
Therefore, by Theorem~\ref{t21}, $\hh$ is weakly PAC learnable with confidence $2/3$, error $1/3$, and $O(1)$ examples.
Now, Theorem~\ref{t22} implies that $\hh$ has a sample compression scheme of size $k(m) \leq O(\log(m)\log\log(m))$.
In other words, the following dichotomy holds:
every hypothesis class $\hh$ either has a sample compression scheme of size $k(m)=O(\log (m)\log\log (m))$, or any sample compression scheme for it has size~$\Omega(m)$.

This dichotomy implies the following compactness property for learnability under the zero/one loss.

\begin{theorem}\label{t24}
Let $d\in\mathbb{N}$, and let $\hh$ be an hypothesis class such that each finite subclass of $\hh$
is learnable with error $1/3$, confidence $2/3$ and $d$ examples. Then $\hh$ is learnable
with error $1/3$, confidence $2/3$ and $O(d\log^2 (d)\log\log(d))$ examples.
\end{theorem}

When $\Y = \{0,1\}$, the theorem follows by the observing that
if every subclass of $\hh$ has VC dimension at most $d$, then
the VC dimension of $\hh$ is at most $d$.
We are not aware of a similar argument that applies for a general label set.
A related challenge, which was posed by~\cite{DBLP:conf/colt/DanielyS14},
is to find a ``combinatorial'' parameter, which captures multiclass learnability
like the VC dimension captures it in the binary-labeled case.

A proof of Theorem~\ref{t24} appears in Section~\ref{sec:t24proof}.
It uses an analogous\footnote{Ben-David and Litman proved a compactness result for sample compression schemes when $\Y=\{0,1\}$, but their argument generalizes for a general $\Y$.}
compactness property for sample compression schemes proven by~\cite{DBLP:journals/dam/Ben-DavidL98}.


\subsubsection{Uniform convergence versus compression to constant size}\label{s224}

Since the introduction of sample compression schemes by~\cite{littleWarm},
they were mostly studied in the context of binary-labeled hypothesis classes (the case $\Y = \{0,1\}$).
In this context, a significant number of works were dedicated to studying the relationship
between VC dimension and the minimal size of a compression scheme (e.g.~\cite{DBLP:conf/colt/Floyd89,DBLP:journals/siamcomp/HelmboldSW92,DBLP:journals/ml/FloydW95,DBLP:journals/dam/Ben-DavidL98,DBLP:journals/jmlr/KuzminW07,chernikovS,DBLP:journals/jcss/RubinsteinBR09,RR3,DBLP:conf/colt/LivniS13}).
Recently,~\cite{DBLP:journals/eccc/MoranY15} proved that any class of VC dimension $d$
has a compression scheme of size {exponential
in the VC dimension}.
Establishing whether a compression scheme of size linear (or even polynomial)
in the VC dimension remains open~\cite{DBLP:journals/ml/FloydW95,DBLP:conf/colt/Warmuth03}.

This question has a natural extension to multiclass categorization:
Does every hypothesis class $\hh$ have a sample compression scheme of size $O(d)$,
where $d=d_{PAC}(1/3,1/3)$ is the minimal sample complexity of a weak learner for $\hh$?
In fact, in the case of multiclass categorization it is open whether there is a sample compression
scheme of size depending only on $d$.

We show here that the arguments from \cite{DBLP:journals/eccc/MoranY15}
generalize to uniform convergence.

\begin{theorem}\label{t25}
Let $\hh$ be an hypothesis class with uniform convergence rate $d^{UC}(\eps,\delta)$.
Then $\hh$ has a sample compression scheme of size $\exp(d)$, where $d=d^{UC}(1/3,1/3)$.
\end{theorem}

The proof of this theorem uses the notion of the graph dimension, which was defined by~\cite{Natarajan89}.
For a function $f:\X\to \Y$ and $h\in\hh$, define $h_{f}:\X\to\left\{ 0,1\right\} $
to be
\[
h_{f}\left(x\right)=\begin{cases}
0\quad & h(x)=f(x),\\
1 & h(x) \neq f(x) .
\end{cases}
\]
Set $\hh_{f}=\left\{ h_{f} : h\in\hh\right\} $. The
{\em graph dimension} $\dim_G(\hh)$ of $\hh$ is $\sup_{f}VC(\hh_{f})$ where $f$ runs over all functions in $\Y^{\X}$
and VC indicates the VC dimension.
Note that for $\Y=\{0,1\}$, the graph dimension is exactly the VC dimension.

Theorem~\ref{t25} is proved using the following two ingredients.
First, the construction in~\cite{DBLP:journals/eccc/MoranY15} yields a sample compression scheme of size $\exp(\dim_G(\hh))$.
Second, the graph dimension determines the uniform convergence rate, similarly to that the VC dimension does it in the binary-labeled case.
\begin{theorem}\label{t36}
Let $\hh$ be an hypothesis class, let $d=\dim_G(\hh)$, and let $d^{UC}(\eps,\delta)$ denote the uniform convergence rate of $\hh$. Then, there exist constants $C_1,C_2$ such that
$$  C_1\cdot \frac{d+\log(1/\delta)-C_1}{\eps^2}\leq d^{UC}(\eps,\delta) \leq C_2 \cdot \frac{d \log(1/\eps) + \log(1/\delta)}{\eps^2}.$$
\end{theorem}
Parts of this result are well-known and appear in the literature: 
The upper bound follows from
Theorem 5 of~\cite{DBLP:journals/corr/DanielySBS13},
and the core idea of the argument dates back to the articles 
of~\cite{BenDavid95} and of~\cite{zbMATH03391742}.
A lower bound with a worse dependence on $\eps$ follows from Theorem 9 of~\cite{DBLP:journals/corr/DanielySBS13}. 
We prove Theorem~\ref{t36} in Section~\ref{app:c26}.
Part of the argument is about proving tight anti-concentration
results for the binomial distribution.

%

\section{General learning setting}
\label{sec:GenLoss}

We have seen that in the case of the zero/one loss function, an existence of a sublinear sample compression scheme is equivalent to learnability. It is natural to ask whether this phenomenon extends to other loss functions. The direction
``compression $\implies$ learning'' remains valid for general loss functions. In contrast, as will be discussed in this section, the other direction fails for general loss functions.

However, a natural adaptation of sample compression schemes, which we term {\em approximate sample compression schemes}, allows the extension of the equivalence to arbitrary loss functions. 
Approximate compression schemes were previously studied in the context of classification (e.g.~\cite{DBLP:journals/ml/GraepelHS05,DBLP:journals/corr/GottliebK15}).
In Subsection~\ref{sec:linreg} we argue that in general sample compression schemes are not equivalent to learnability; specifically, there is no agnostic sample compression scheme for linear regression.
In Subsection~\ref{sec:approx} we define approximate sample compression schemes and establish their equivalence with learnability.

Finally, in Subsection~\ref{sec:separation} we use this equivalence to demonstrate classes that are PAC learnable but not agnostic PAC learnable. This manifests a difference with the zero/one loss function under which agnostic and PAC learning are equivalent (see~\ref{sec:PACeqAGN}). It is worth noting that the loss function we use to break the equivalence takes only three values
(compared to the two values of the zero/one loss function).

\subsection{No agnostic compression for linear regression}\label{sec:linreg}

We next show that in the setup of linear regression, which is
known to be agnostic PAC learnable, there is no agnostic sample compression scheme.
For convenience, we shall restrict the discussion to zero-dimensional linear regression.
In this setup\footnote{One may think of $X$ as a singleton.}, the sample consists of $m$ examples $S = (z_1,z_2,\ldots,z_m) \in[0,1]^m$, and the loss function
is defined by $\ell(h,z) = (h-z)^2$.
The goal is to find $h\in \mathbb{R}$ which minimizes $L_S(h)$.
The empirical risk minimizer (ERM) is exactly the average $h^* = \frac{1}{m}\sum_i z_i$,
and for every $h \neq h^*$ we have $L_S(h) > L_S(h^*)$. Thus, an agnostic sample compression scheme in this setup should compress
$S$ to a subsequence and a binary string of side information, from which the average of $S$ can be reconstructed.
We prove that there is no such compression.

\begin{theorem}\label{t:nocompression}
There is no agnostic sample compression scheme for zero-dimensional linear regression with size $k(m)\leq m/2$.
\end{theorem}

The proof idea is to restrict our attention to sets $\Omega\subseteq [0,1]$ for which every subset of $\Omega$ has a distinct average. It follows that any sample compression scheme for samples from $\Omega$ must
perform a compression that is information theoretically impossible.

\begin{proof}
Let $\Omega\subseteq [0,1]$ be a set of linearly independent numbers over $\mathbb{Q}$ of cardinality $M$.
Thus, for every two distinct subsets $A,B\subseteq \Omega$, the averages $a,b$ of the numbers in $A,B$ are distinct (otherwise a non trivial linear dependence over $\mathbb{Q}$ is implied).
It follows that there are $\binom{M}{m}$ distinct averages of sets of size $m$. On the other hand, the size of the image of $\kappa$
on such inputs is at most
${M \choose k(m)} \cdot 2^{k(m)}$.
Thus, for a sufficiently large $M$  there is no agnostic sample compression scheme of size $k(m)\leq m/2$.
\end{proof}

\subsection{Approximate sample compression schemes}\label{sec:approx}

The previous example 
suggests the question
of whether one can generalize the definition of compression to fit  problems where the loss function is not zero/one.
Taking cues from PAC and agnostic PAC learning, we consider the following definition.
We say that the selection scheme $(\kappa,\rho)$ is an {\em $\eps$-approximate} sample compression scheme for $\hh$ if
for every sample $S$ that is realizable by $\hh$,
$$L_{S}\left(\rho\left(\kappa\left(S\right)\right)\right) \leq\eps.$$
It is called an {\em $\eps$-approximate} agnostic sample compression scheme for $\hh$ if
for every sample $S$,
$$
L_{S}\left(\rho\left(\kappa\left(S\right)\right)\right) \leq \inf_{h\in\hh}L_{S}(h) +\eps.
$$

Let us start by revisiting the case of zero-dimensional linear regression.
Even though it does not have an agnostic compression scheme of sublinear size,
it does have an $\eps$-approximate agnostic sample compression scheme of size $k=O(\log(1/\eps)/\eps)$ which we now describe.

Given a sample $S=(z_1 ,\ldots, z_m)\in [0,1]$,
the average $h^*=\sum_{i=1}^{m}{z_i}/m$ is
the ERM of $S$.
Let
$$L^*=L(h^*)=\sum_{i=1}^{m}{z_i^2}/m - \left(\sum_{i=1}^{m}{z_i}/m\right)^2.$$
It is enough to show that there exists a sub-sample $S'= (z_{i_1},\ldots,z_{i_{\ell}})$ of size $\ell = \lceil1/\eps\rceil$ such that
$$L_S \left(\sum_{j=1}^{\ell}{z_{i_j}}/\ell \right)\leq L^* + \eps.$$
It turns out that picking $S'$ at random suffices.
Let $Z_1,\ldots,Z_{\ell}$ be independent random variables that are uniformly distributed over $S$ and let $H=\frac{1}{\ell}\sum_{i=1}^{\ell}{Z_i}$ be their average.
Thus, $\E[H] = h^*$ and
$$\E[L_S(H)] = L^* + \Var[H] \leq L^* + \eps.$$
In particular, this means that there exists some sub-sample of size $\ell$ whose average has loss at most $L^*+\eps$.
Encoding such a sub-sample requires $O(\log(1/\eps)/\eps )$ additional bits of side information.

We now establish the equivalence between approximate compression and learning (the proof is similar to the proof of Theorem~\ref{t21}).

\begin{theorem}[Approximate compressing implies learning]
Let $(\kappa,\rho)$ be a {\precomp} of size $k$, let $\hh$ be an hypothesis class,
and let $\D$ be a distribution on $\Z$.
\begin{enumerate}
\item If $(\kappa,\rho)$ is an $\eps$-approximate sample compression scheme for $\hh$, and $m$ is such that $k(m)\leq m/2$, then
$$\Pr_{S\sim \D^m}\left(L_\D\left(\rho\left(\kappa\left(S\right)\right)\right)> \eps + 100\sqrt{\frac{k\log \frac{m}{k} +\log\frac{1}{\delta}}{m}}\right)<\delta.$$
\item If $(\kappa,\rho)$ is an $\eps$-approximate agnostic sample compression scheme for $\hh$, and $m$ is such that $k(m)\leq m/2$, then
$$\Pr_{S\sim \D^m}\left(L_\D\left(\rho\left(\kappa\left(S\right)\right)\right)> \inf_{h\in \hh}L_\D(h) + \eps + 100\sqrt{\frac{k\log \frac{m}{k}+\log\frac{1}{\delta}}{m}}\right)<\delta.$$
\end{enumerate}
\end{theorem}



The following Theorem shows that every learnable class has an approximate sample compression scheme. The proof of this theorem is straightforward | in contrast with the proof of the analog statement in the case of zero/one loss functions and compression schemes
without error.

\begin{theorem}[Learning implies approximate compressing]
\label{t33}
Let $\hh$ be an hypothesis class.
\begin{enumerate}
\item If $\hh$ is PAC learnable with rate $d(\eps,\delta)$, then it has an $\eps$-approximate sample compression scheme of size $k \leq O(d \log(d))$ with $d = \min_{\delta<1}{d(\eps,\delta)}$.
\item If $\hh$ is agnostic PAC learnable with rate $d(\eps,\delta)$, then it has an $\eps$-approximate agnostic sample compression scheme of size $k \leq O(d \log(d))$ with $d = \min_{\delta<1}{d(\eps,\delta)}$.
\end{enumerate}
\end{theorem}
\begin{proof}
We prove the first item, the second item can be proven similarly.
Fix some $\eps,\delta>0$ and a sample $S=(z_1,\ldots,z_m)\in \Z^m$ with $m>d(\eps,\delta)$. Let $\D$ be the uniform probability on $S$. Since $A$ is a learning algorithm,
a fraction of $1-\delta$ of the sub-samples $S'$ of $S$ of size $d$ satisfy that $L_\D(A(S'))<\eps$.
In particular, there is at least one such tuple $S'$ if $\delta < 1$. Thus, the compression function picks this sub-sample $S'$, and the reconstruction function reconstruct the hypothesis $A(S')$ using the learning algorithm.
The sub-sample $S'$ may have multiplicities,
which are encoded using the side information
with at most $O(\log(d))$ bits per element of $S'$.
The overall size is at most $O(d \log(d))$.
\end{proof}

\subsection{A separation between PAC and agnostic learnability}\label{sec:separation}
Here we establish a separation between PAC and agnostic PAC learning under loss functions which take more than two values.
We first construct hypothesis classes $\hh$ for which the gap between the PAC and agnostic learning rates is arbitrarily large, and later construct a single hypothesis class which is PAC learnable and not agnostic PAC learnable.

{A central ingredient in our proof
is Ramsey theory for hypergraphs (see, e.g.~\cite{erdos2011combinatorial}).}

%
%
%

We begin with defining the hypothesis class.
Fix $M,K \in \mathbb{N}$ with $M$ sufficiently larger than $K$. Let $\X=\{1,\ldots,M\}$
and let $\Y$ be the family of all subsets of $\X$ of size at most $K$.
Let $\hh$ be the hypothesis class of all constant functions from $\X$ to $\Y$.
Define a loss function on $\Y$ as follows:
$$\ell\left(A,B\right)=\begin{cases}
0 & A=B,\\
1/2 & A\neq B,\;A\cap B\neq\emptyset,\\
1 & A\cap B=\emptyset.
\end{cases}$$

First observe that $\hh$ is PAC learnable,
since it has a sample compression scheme of size one;
any realizable sample is of the form $((x_1,A),\ldots,(x_m,A))$ for some $A\in Y$, and therefore can be compressed to
$(x_1,A)$. The reconstruction function then takes $(x_1,A)$ to the constant hypothesis $h \equiv A$.

To show that the agnostic sample complexity of $\hh$ is large, we show that any approximate agnostic compression scheme for it is large.
This is another demonstration of the usefulness
of the notion of approximate compression schemes.

\begin{lemma}\label{l34}
Any $1/4$-approximate agnostic compression scheme for $\hh$ has size at least $K/2$ on some sample of length $K$.
\end{lemma}

\begin{proof}
Let $(\kappa,\rho)$ be a selection scheme. Assume that the size $k=k(K)$ of $(\kappa,\rho)$ is less than $K/2$. We will show that $(\kappa,\rho)$ is not a $1/4$-approximate agnostic compression scheme for $\hh$.

 Consider samples of the form $S=((x,\{a_1\}),\ldots,(x,\{a_K\}))$ for some $x$ and $a_1<\cdots<a_K$. The constant function which takes the value $A=\{a_1,\ldots,a_K\}$ has empirical risk $1/2$, because $A$ intersects each set in the sample.
It thus suffices to show that
the empirical risk of $\rho(\kappa(S))$ on some sample
$S$ of this form is at least $3/4$.

Consider the complete $K$-regular hypergraph on the vertex set $\X$. We use $(\kappa,\rho)$ to color the hyperedges as follows. Given an hyperedge $A=\{a_1<\ldots<a_K\}$,
consider the sample $$S_A = ((x,\{a_1\}),\ldots,(x,\{a_K\})).$$
Its compression $\kappa(S_A)$
yields some subset $\{a_{i_1},\ldots,a_{i_T}\}$ of $A$ with $T\leq k$
and some bit string of side information $b$. Define the color of $A$ to be
$$c(A) = (\{i_1,\ldots,i_T\}, b).$$
The number of colors is at most $4^k$.
By assumption on the size of $(\kappa,\rho)$,
we know that $T \leq K/2$.

By Ramsey's theorem, if $M$ is sufficiently large, then there is monochromatic component $C\subseteq \X$ of size at least $(2K+2)\cdot(K/2 + 1)$.
In other words, every $A\subseteq C$ of size $K$ is compressed to the same $T$ coordinates with the same side information $b$.
Let $i_1,\ldots,i_T$ denote these coordinates.

Since $|C|$ is large, we can pick $T$ points $A'=\{a'_1<\ldots <a'_T\}\subset C$ that are far away from each other inside $C$:
$$|(0,a'_1)\cap C|,|(a'_1,a'_2)\cap C|,|(a'_2,a'_3)\cap C|,\ldots,|(a'_T,M+1)\cap C| > 2K.$$
Let
$$h = \rho(((x,\{a'_1\}),\ldots,(x,\{a'_T\})),b),$$
and let $R=h(x) \in \Y$.
By the choice of $A'$, since $|R| \leq K$, we can find some hyperedge $A=\{a_1<\ldots<a_K\} \subset C$ such that
$$\{a_{i_1},\ldots,a_{i_T}\}=A' = R \cap A.$$
It follows that for $S=((x,\{a_1\}),\ldots,(x,\{a_K\}))$,
we have $\rho(\kappa(S))=h$, and
$$L_S(h) = \frac{1}{K} \left( T \cdot \frac{1}{2}
+ (K-T) \cdot 1 \right) \geq \frac{3}{4}.$$
\end{proof}

There are several ways to use this construction to obtain an infinite PAC learnable class $\hh$
that has no $1/4$-approximate agnostic compression scheme (and therefore is not agnostic PAC learnable).
One option follows by taking $\X$ to be a sufficiently large infinite set, $\Y$ to be all
finite subsets of $\X$, and the same loss function.
The proof that the resulting class has no $1/4$-approximate agnostic compression scheme
is an adaptation of the proof of Lemma~\ref{l34} using infinite Ramsey theory:
the cardinality of $\X$ is chosen so that the following property holds: for every $n\in\mathbb{N}$, and every coloring
of the $n$-elements subsets of $\X$ in a finite number of colors, there is some uncountable monochromatic
subset of $\X$. Such $\X$ exists due to a more general result by Erd\"{o}s and Rado~\cite{erdos56}.
However, this result requires $\X$ to be huge; in particular, larger than any finite tower of the form  $2^{\cdot^{\cdot^{\cdot^{2^{\aleph_0}}}}}$.
It is interesting to note that if $\X$ is countable, say $\X=\mathbb{N}$, then the resulting $\hh$ is agnostic PAC learnable
by an agnostic sample compression scheme of size 1: if the input sample is not realizable then output the constant function that takes the value $\{0,\ldots,M\}$ where $M$ is the maximum number observed in the input sample. This set intersects every non empty set in the sample and therefore minimizes the empirical risk.
To conclude, if $\X$ is countable then $\hh$ is agnostic PAC learnable, and if $\X$ is huge (larger than any finite tower of the form $2^{\cdot^{\cdot^{\cdot^{2^{\aleph_0}}}}}$) then $\hh$ is not agnostic PAC learnable. It will be interesting to determine the
minimum cardinality of $\X$ for which $\hh$ is not agnostic PAC learnable.

One may find the usage of such huge cardinals unnatural in this context.
Therefore, we present an alternative approach of using Lemma~\ref{l34} to construct a countable PAC learnable class
that is not agnostic PA learnable. This uses the following general construction:
given hypothesis classes $\hh_1,\hh_2,\ldots$ with mutually disjoint domains $\X_1,\X_2,\ldots$ and mutually disjoint ranges $\Y_1,\Y_2,\ldots$ define an hypothesis class $\hh$ that is the union of the $\hh_i$'s.
To this end, extend the domain of each $h_i\in\hh_i$ to $\cup_{i}X_i$ by setting $h_i(x)=b$ for all $x\notin X_i$, where $b$ is 
some dummy element that does not belong to $\cup_i{Y_i}$. The loss function is extended so that $\ell(y_i,y_j)=1$ for $y_i,y_j$ from different copies, and $\ell(b,y)=1$ for all $y\neq b$.
Thus $\hh=\cup_i{\hh_i}$ is an hypothesis class with domain  
$\cup_i{\X_i}$ and range $\cup_i{\Y_i}\cup\{b\}$. One can verify that (i) if each $\hh_i$ has a realizable-case sample compression 
scheme of size $d$ then also $\hh$ has such a scheme, and that (ii) if $\hh$ has an approximate agnostic sample compression scheme of size $d$ then also each $\hh_i$ has such a scheme. Thus, picking the $\hh_i$'s according to Lemma~\ref{l34} so that $\hh_i$ has no $1/4$-
approx.\ sample compression scheme of size $i$, yields a countable PAC learnable class $\hh$ that is not agnostic PAC learnable.

\section{Discussion and further research}
The compressibility-learnability equivalence is a fundamental
link in statistical learning theory. 
From a theoretical perspective this link can serve as a guideline
for proving both negative/impossibility results, and positive/possibility results.

From the perspective of positive results,  
just recently,~\cite{DBLP:conf/colt/CummingsLNRW16} relied on this paper
in showing that every learnable problem is learnable with robust generalization guarantees.
Another important 
example appears in the work of 
boosting weak learners~\cite{schapire2012boosting} (see Chapter 4.2). 
These works follow a similar approach, that may be useful in other scenarios: 
(i) transform the given learner to a sample compression scheme, 
and (ii) utilize properties of compression schemes to derive the desired result.
The same approach is also used in this paper in Section~\ref{sec:PACeqAGN},
where it is shown that PAC learning implies agnostic PAC learning
under 0/1 loss; we first transform the PAC learner to a realizable compression scheme, 
and then use the realizable compression scheme to get an agnostic compression scheme
that is also an agnostic learner. 
We note that we are not aware of a proof that directly
transforms the PAC learner to an agnostic learner without using compression.

From the perspective of impossibility/hardness results, this link implies that
to show that a problem is not learnable, it suffices to show that it is not compressible. 
In Section~\ref{sec:separation},  
we follow this approach when showing that PAC and agnostic PAC learnability are not
equivalent for general loss functions. 

This link may also have a practical impact,
since it offers a thumb rule for algorithm designers;
if a problem is learnable then it can be learned by a compression algorithm, 
whose design boils down to an intuitive principle 
``find a small insightful subset of the input data.''
For example, in geometrical problems, 
this insightful subset often appears on the boundary of the data points
(see e.g.\ \cite{DBLP:journals/corr/GottliebK15}).

\section*{Acknowledgements}
{We thank Noga Alon for suggesting the outline of the proof of Theorem~\ref{td1} to us.}
\bibliographystyle{plain}
\bibliography{compRef}

\begin{thebibliography}{10}

\bibitem{BenDavid95}
S.~Ben{-}David, N.~Cesa{-}Bianchi, D.~Haussler, and P.~M. Long.
\newblock Characterizations of learnability for classes of \{0,...,n\}-valued
  functions.
\newblock {\em J. Comput. Syst. Sci.}, 50(1):74--86, 1995.

\bibitem{DBLP:journals/dam/Ben-DavidL98}
Shai Ben{-}David and Ami Litman.
\newblock {Combinatorial Variability of {Vapnik-Chervonenkis} Classes with
  Applications to Sample Compression Schemes}.
\newblock {\em Discrete Applied Mathematics}, 86(1):3--25, 1998.

\bibitem{BerKont}
Daniel Berend and Aryeh Kontorovich.
\newblock A sharp estimate of the binomial mean absolute deviation with
  applications.
\newblock {\em Statistics \& Probability Letters}, 83(4):1254--1259, 2013.

\bibitem{zbMATH04143473}
Anselm Blumer, Andrzej Ehrenfeucht, David Haussler, and Manfred~K. Warmuth.
\newblock {Learnability and the Vapnik-Chervonenkis dimension.}
\newblock {\em {J. Assoc. Comput. Mach.}}, 36(4):929--965, 1989.

\bibitem{opac-b1134861}
St{\'e}phane Boucheron, G{\'a}bor Lugosi, and Pascal Massart.
\newblock {\em Concentration inequalities : a nonasymptotic theory of
  independence}.
\newblock Oxford university press, Oxford, 2013.

\bibitem{chernikovS}
A.~Chernikov and P.~Simon.
\newblock Externally definable sets and dependent pairs.
\newblock {\em Israel Journal of Mathematics}, 194(1):409--425, 2013.

\bibitem{DBLP:conf/colt/CummingsLNRW16}
Rachel Cummings, Katrina Ligett, Kobbi Nissim, Aaron Roth, and Zhiwei~Steven
  Wu.
\newblock Adaptive learning with robust generalization guarantees.
\newblock In {\em Proceedings of the 29th Conference on Learning Theory, {COLT}
  2016, New York, USA, June 23-26, 2016}, pages 772--814, 2016.

\bibitem{DBLP:conf/colt/DanielyS14}
A.~Daniely and S.~Shalev{-}Shwartz.
\newblock Optimal learners for multiclass problems.
\newblock In {\em COLT}, volume~35, pages 287--316, 2014.

\bibitem{DBLP:journals/corr/DanielySBS13}
Amit Daniely, Sivan Sabato, Shai Ben-David, and Shai Shalev-Shwartz.
\newblock Multiclass learnability and the {ERM} principle.
\newblock {\em Journal of Machine Learning Research}, 16:2377--2404, 2015.

\bibitem{erdos2011combinatorial}
P.~Erd{\"o}s, A.~M{\'a}t{\'e}, A.~Hajnal, and P.~Rado.
\newblock {\em Combinatorial Set Theory: Partition Relations for Cardinals:
  Partition Relations for Cardinals}.
\newblock Studies in Logic and the Foundations of Mathematics. Elsevier
  Science, 2011.

\bibitem{erdos56}
P.~Erd{\"o}s and R.~Rado.
\newblock A partition calculus in set theory.
\newblock {\em Bull. Amer. Math. Soc.}, 62(5):427--489, 09 1956.

\bibitem{DBLP:conf/colt/Floyd89}
S.~Floyd.
\newblock {Space-Bounded Learning and the Vapnik-Chervonenkis Dimension}.
\newblock In {\em COLT}, pages 349--364, 1989.

\bibitem{DBLP:journals/ml/FloydW95}
Sally Floyd and Manfred~K. Warmuth.
\newblock {Sample Compression, Learnability, and the Vapnik-Chervonenkis
  Dimension}.
\newblock {\em Machine Learning}, 21(3):269--304, 1995.

\bibitem{DBLP:journals/iandc/Freund95}
Yoav Freund.
\newblock Boosting a weak learning algorithm by majority.
\newblock {\em Inf. Comput.}, 121(2):256--285, 1995.

\bibitem{schapire2012boosting}
Yoav Freund and Robert~E. Schapire.
\newblock {\em Boosting: Foundations and Algorithms}.
\newblock Adaptive computation and machine learning. MIT Press, 2012.

\bibitem{DBLP:journals/corr/GottliebK15}
Lee{-}Ad Gottlieb, Aryeh Kontorovich, and Pinhas Nisnevitch.
\newblock Nearly optimal classification for semimetrics.
\newblock {\em CoRR}, abs/1502.06208, 2015.

\bibitem{DBLP:journals/ml/GraepelHS05}
Thore Graepel, Ralf Herbrich, and John Shawe{-}Taylor.
\newblock {PAC-Bayesian Compression Bounds on the Prediction Error of Learning
  Algorithms for Classification}.
\newblock {\em Machine Learning}, 59(1-2):55--76, 2005.

\bibitem{DBLP:journals/siamcomp/HelmboldSW92}
D.~P. Helmbold, R.~H. Sloan, and M.~K. Warmuth.
\newblock Learning integer lattices.
\newblock {\em {SIAM} J. Comput.}, 21(2):240--266, 1992.

\bibitem{KB80}
R.~Kaas and J.M. Buhrman.
\newblock Mean, median and mode in binomial distributions.
\newblock {\em Statistica Neerlandica}, 34(1):13--18, 1980.

\bibitem{DBLP:journals/jmlr/KuzminW07}
Dima Kuzmin and Manfred~K. Warmuth.
\newblock Unlabeled compression schemes for maximum classes.
\newblock {\em Journal of Machine Learning Research}, 8:2047--2081, 2007.

\bibitem{littleWarm}
Nick Littlestone and Manfred Warmuth.
\newblock Relating data compression and learnability.
\newblock {\em Unpublished}, 1986.

\bibitem{DBLP:conf/colt/LivniS13}
Roi Livni and Pierre Simon.
\newblock Honest compressions and their application to compression schemes.
\newblock In {\em COLT}, pages 77--92, 2013.

\bibitem{DBLP:journals/eccc/MoranY15}
Shay Moran and Amir Yehudayoff.
\newblock Sample compression schemes for {VC} classes.
\newblock {\em J. ACM}, 63(3):21:1--21:10, June 2016.

\bibitem{Natarajan89}
B.~K. Natarajan.
\newblock On learning sets and functions.
\newblock {\em Machine Learning}, 4:67--97, 1989.

\bibitem{Neumann1928}
J.~von Neumann.
\newblock Zur theorie der gesellschaftsspiele.
\newblock {\em Mathematische Annalen}, 100:295--320, 1928.

\bibitem{RR3}
B.~I.~P. Rubinstein and J.~H. Rubinstein.
\newblock A geometric approach to sample compression.
\newblock {\em Journal of Machine Learning Research}, 13:1221--1261, 2012.

\bibitem{DBLP:journals/jcss/RubinsteinBR09}
Benjamin I.~P. Rubinstein, Peter~L. Bartlett, and J.~H. Rubinstein.
\newblock Shifting: One-inclusion mistake bounds and sample compression.
\newblock {\em J. Comput. Syst. Sci.}, 75(1):37--59, 2009.

\bibitem{Shalev-Shwartz}
Shai Shalev-Shwartz and Shai Ben-David.
\newblock {\em Understanding Machine Learning: From Theory to Algorithms}.
\newblock Cambridge University Press, New York, NY, USA, 2014.

\bibitem{DBLP:books/daglib/0097035}
Vladimir Vapnik.
\newblock {\em Statistical learning theory}.
\newblock Wiley, 1998.

\bibitem{zbMATH03391742}
V.N. {Vapnik} and A.Ya. {Chervonenkis}.
\newblock {On the uniform convergence of relative frequencies of events to
  their probabilities.}
\newblock {\em {Theory Probab. Appl.}}, 16:264--280, 1971.

\bibitem{DBLP:conf/colt/Warmuth03}
Manfred~K. Warmuth.
\newblock Compressing to {VC} dimension many points.
\newblock In {\em COLT/Kernel}, pages 743--744, 2003.

\end{thebibliography}

\newpage
\appendix

\section{Selection schemes do not overfit}\label{s:appt11}
{Let $(\kappa,\rho)$ be a selection scheme. Let $S=(Z_1,\ldots,Z_m)$} be the input sample. For every bit string $b$ of length at most $k$ and $T\subseteq [m]$ of size at most $k(m)$,
let $h_{T,b}=\rho(\{Z_i:i\in T\},b)$. The proof strategy is to show that a fixed $h_{T,b}$ overfits  with very small probability.
Then, since $A(S)\in\{h_{T,b}:b\in \{0,1\}^{\leq k}, T\subseteq [m], |T|\leq k\}$, an application of the union bound will imply that $A(S)$ overfits with a small probability.%

The following lemma shows that the probability that $h_{T,b}$ overfits is small.
\begin{lemma}\label{l13}
For every $h=h_{T,b}$, $\delta'>0$:
$$\Pr_{S \sim \D} \left[\lvert L_\D(h) - L_{S}(h)\rvert\geq  \sqrt{\frac{8L_{S}(h)\log(1/\delta')}{m}} +\frac{16\log(1/\delta')+k}{m}\right]\leq\delta'.$$
\end{lemma}
Before proving this lemma, we use it together with a simple application of a union bound in order to prove Theorem~\ref{t11}.
\begin{proof}[Proof of Theorem ~\ref{t11}]
Set $\eps=\sqrt{\frac{16L_{S}(h)\log(1/\delta')}{m}} +\frac{16\log(1/\delta')+2k}{m}$.
\begin{align*}
\Pr_{S\sim \D^m}\left(\lvert L_\D(A(S)) -  L_S(A(S))\rvert \geq  \eps\right)
&\leq
\Pr_{S\sim \D^m}\left(\exists T,b: \lvert L_\D(h_{T,b}) -  L_S(h_{T,b})\rvert \geq  \eps\right)\\
&\leq
\left\lvert\{0,1\}^{\leq k}\right\rvert\cdot\left\lvert {m \choose \leq k} \right\rvert\cdot \Pr_{S\sim \D^m}\left[\lvert L_\D(h)- L_S(h)\rvert\geq\eps\right]\tag{the union bound}\\
&\leq
2^{k+1}\left(\frac{em}{k}\right)^k\delta'.\tag{Lemma~\ref{l13}}
\end{align*}
Plugging in $\delta'=2^{-k-1}\left(\frac{em}{k}\right)^{-k}\delta$ yields the desired inequality and finishes the proof of Theorem~\ref{t11}.
\hfill
\end{proof}

For the proof of {Lemma~\ref{l13}}, we use the following lemma from~\cite{Shalev-Shwartz}.
\begin{lemma}[Lemma B.10 in~\cite{Shalev-Shwartz}]\label{l92}
For every fixed hypothesis $h$, every $\delta'>0$, and $m\in\mathbb{N}$:
\begin{enumerate}
\item
$\Pr_{S\sim \D^m}\left[L_S(h) \geq L_\D(h) + \sqrt{\frac{2L_S(h)\log(1/\delta')}{3m}} +\frac{2\log(1/\delta')}{m}\right]\leq \delta'.$
\item
$\Pr_{S\sim \D^m}\left[L_\D(h) \geq L_S(h) + \sqrt{\frac{2L_S(h)\log(1/\delta')}{m}} +\frac{4\log(1/\delta')}{m}\right]\leq \delta'.$
\end{enumerate}
\end{lemma}
In particular, by plugging $\delta'/2$ instead of $\delta'$ and applying a union bound, we get
$$\Pr_{S\sim \D^m}\left[\lvert L_\D(h) - L_S(h)\rvert \geq \sqrt{\frac{4L_S(h)\log(1/\delta')}{m}} +\frac{8\log(1/\delta')}{m}\right]\leq \delta'.$$

\begin{proof}[Proof of {Lemma~\ref{l13}}]
Fix some $h=h_{T,b}$.
Imagine that we sample $S$ as follows. First sample the $k$ examples of $T$ and construct $h$ according to them and $b$.
Next sample the remaining $m-k$ examples in $S':=S\setminus T$ and calculate the empirical risk of $h$. Since $h$ is independent on the remaining $m-k$ samples in $S'$, by Lemma~\ref{l92} it follows that:
\begin{equation}\label{eq:91}
\Pr_{S'\sim \D^{m-k}}\left[\lvert L_\D(h) - L_{S'}(h)\rvert \geq \sqrt{\frac{4L_{S'}(h)\log(1/\delta')}{m-k}} +\frac{8\log(1/\delta')}{m-k}\right]\leq \delta'.
\end{equation}
We are left with the task of changing $S'$ into $S$ in the expression above. It is easily seen that
$$(m-k)L_{S'}(h)\leq m L_S(h) \leq (m-k)L_{S'}(h)+k$$ which implies that
$$\frac{-k}{m}L_{S'}(h) \leq L_S(h)-L_{S'}(h) \leq \frac{-k}{m}L_{S'}(h)+\frac{k}{m}.$$
Using the fact that $0\leq L_{S'}(h) \leq 1$, we conclude that $|L_S(h)-L_{S'}(h)|\leq \frac{k}{m}$, and hence $|L_D(h)-L_{S'}(h)|\geq |L_D(h)-L_S(h)|-\frac{k}{m}$.
Let $E_1$ denote the event that
$$\lvert L_\D(h) - L_{S'}(h)\rvert\geq  \sqrt{\frac{4L_{S'}(h)\log(1/\delta')}{m-k}} +\frac{8\log(1/\delta')}{m-k},$$
and $E_2$ denote the event that
$$\lvert L_\D(h) - L_{S}(h)\rvert\geq \sqrt{\frac{8L_{S'}(h)\log(1/\delta')}{m}} +\frac{16\log(1/\delta')+k}{m}.$$
Since $k\leq \frac{m}{2}$, it follows that $E_2 \subseteq E_1$ and therefore $P_{S\sim \D^m}(E_2)\leq P_{S'\sim \D^{m-k}}(E_1)\leq \delta'$ which completes the proof.
\end{proof}

\section{Proof of Lemma~\ref{l22}}\label{app:l22}
By the definition of agnostic sample compression schemes
we have that the empirical risk of $A(S)$ is lesser than or equal to the empirical risk of every $h\in \hh$:
\begin{equation}\label{eq:53}
L_S(A(S))\leq\inf_{h\in \hh}L_S(h).
\end{equation}
Since $S\sim D^m$, it follows that for every $h\in \hh$,
$L_S(h) = \frac{Z_1+\ldots Z_m}{m}$,
where the $Z_i$'s are i.i.d~random variables distributed over $[0,1]$ with expectation $L_\D(h)$.
Now, the Chernoff-Hoeffding inequality implies that for $\eps_1'(m,\delta) = \sqrt{\frac{\log\frac{1}{\delta}}{2m}}$
\begin{equation}\label{eq:54}
\Pr_{S\sim D^m}\left(L_S(h)\geq L_\D(h)+\eps_1(m,\delta)\right)\leq\delta.
\end{equation}
Since $\eps_1(m,\delta) > \eps_1'(m,\delta)$ there is $h'\in \hh$ such that
\begin{equation}\label{eq:55}
L_\D(h')\leq \inf_{h\in C} L_\D(h) + \eps_1(m,\delta)-\eps_1'(m,\delta).
\end{equation}
It follows that:
\begin{align*}
\Pr_{S\sim D^m}\left(L_S(A(S))\geq \inf_{h\in \hh}L_\D(h)+\eps_1(m,\delta)\right)
&\leq
\Pr_{S\sim D^m}\left(\inf_{h\in \hh}L_S(h))\geq \inf_{h\in \hh}L_\D(h)+\eps_1(m,\delta)\right)\tag{by Equation~\ref{eq:53}}\\
&\leq
\Pr_{S\sim D^m}\left(L_S(h')\geq \inf_{h\in \hh}L_\D(h)+\eps_1(m,\delta)\right)\tag{since $h'\in \hh$}\\
&\leq
\Pr_{S\sim D^m}\left(L_S(h')\geq L_\D(h') +\eps_1'(m,\delta)\right)\tag{by Equation~\ref{eq:55}}\\
&\leq\delta\tag{by Equation~\ref{eq:54}}
\end{align*}
\hfill \qed

\section{Learning implies compression}\label{app:t23}
We prove here the following result
\begin{theorem*}[\cite{DBLP:journals/iandc/Freund95}]
If $\hh$ is PAC learnable with learning rate $d(\eps,\delta)$, then it has a sample compression scheme of size
$$k(m)=O(d\log (m) \log\log (m) + d \log(m)\log(d)),$$
where $d=d(1/3,1/3)$.
\end{theorem*}
\begin{proof}
Let $A$ be a learner that learns $\hh$ using $d$ examples with
error $\frac{1}{3}$ and confidence $\frac{2}{3}$.

\paragraph{Compression.}
Let $S=\{(x_i,y_i)\}_{i=1}^{m}$ be a realizable sample. 
{Let ${\cal H}_{S}$ be the set of all hypotheses of the
form $A(S')$, where $S'$ is a sample of size $d$
such that each example in $S'$ appears in $S$.}

By the choice of $A$,
for every distribution $\D$ {on $S$}, there is  $h \in {\cal H_S}$ so that
$$L_\D(h)\leq\frac{1}{3}.$$

{
Consider the zero-sum game between the learner $A$
and an adversary in which the learner picks some hypothesis in $\mathcal{H}_S$,
the adversary picks an example in $S$, and the learner pays 
the adversary $1$ unit if the hypothesis is inconsistent with the example, 
and $0$ if it is consistent.
By the previous paragraph, for every mixed strategy of the adversary, 
there is some pure strategy of the learner for which the loss of the learner is at most $\frac{1}{3}$. 
Therefore, by the minimax Theorem~\cite{Neumann1928} 
there is a distribution $p$}
on ${\cal H_S}$ (a mixed strategy of the learner) such that for every $i\leq m$ (pure strategy of the adversary):
\begin{align*}
p(\{h \in {\cal H_S} : h(x_i) = y_i\}) \geq 2/3.
\end{align*}
Let
$h_1,\ldots,h_T$ be $T$ independent samples from $p$,
for $T= 20\log m$.
The constant $20$ is chosen so the Chernoff-Hoeffding inequality yields that for any $i\leq m$ the following event, $E_i$, has probability less than $\frac{1}{m}$
\begin{align*}
\left|\frac{|\{ t \in [T] : h_t(x_i) = y_i) \}|}{T} - p(\{h \in {\cal H_S} : h(x_i) = y_i \})\right| > \frac{1}{6}.
\end{align*}

So, by the union bound: $\Pr(\cup_{i\leq m}E_i)< 1$, and therefore
there is a multiset $F  = \{h_1,h_2,\ldots,h_T\} \subseteq {\cal H_S}$
so that for every $i \leq m$,

\begin{align*}
\frac{|\{ t \in [T] : h_t(x_i) = y_i \}|}{T}
> p(\{h \in {\cal H_S} : h(x_i) = y_i \}) - 1/6 \geq 1/2.
\end{align*}
For every $t \in [T]$, let
$S_t$ be a sub-sample of $S$ of size at most $d$ so that
\begin{align*}
A(S_t) = h_t.
\end{align*}

{$\kappa(S)$ compresses $S$ to the sub-sample $S'$ such that:
\[S' = \bigl\{(x_j,y_j) : \exists t\leq T: (x_j,y_j)\in S_t\bigr\}.\]}
The side information $i \in I$ allows to recover
the sub-samples $S_1,\ldots,S_T$ from $S'$.
One way to encode this information is by a bipartite graph
with one side corresponding to $T$ and the other to $S'$
such that each vertex $t\in T$ is connected to the
vertices in $S$ that form the sub-sample $S_t$.
Since each vertex in $T$ has at most $d$ neighbours in $S'$,
it follows that there are at most $(|S'|^d)^{|T|}=|S'|^{d\cdot|T|}$ such graphs.
Therefore, we can encode such a graph using a bit string $b$ of length $\log|S'|^{d\cdot|T|}$, and
$$|b| = \log|S'|^{d\cdot|T|}\leq O(d\log m \log\log m + d\log m\log d).$$
So, the total size $|b|+|S'|$ is also at most $O(d\log m \log\log m + d\log m\log d)$.

In the reconstruction phase, $i$ and $S'$ are used to decode
the samples $S_t$ for $t\in [T]$. Then, using the learner $A$, the
$h_t$'s are decoded and the reconstructed hypothesis $h$ is defined as their point-wise majority.
By construction $L_S(h)=0$ as required.
\end{proof}

\section{Compactness}\label{sec:t24proof}
Let $\hh\subseteq\Y^\X$ be an hypothesis class such that every finite subclass of it is learnable with confidence $2/3$,
error $1/3$, and $d$ examples. By Theorem~\ref{t22}, $\hh$ has a sample compression scheme of size
$k(m) = O(d\log(m)\log\log(m) + d\log(m)\log(d))$. Thus, to show that $\hh$
is learnable with confidence $2/3$, error $1/3$, and $O(d\log^2(d))$ examples, it suffices to show that $\hh$
has a sample compression scheme of size $k(m)$ (see Section~\ref{sec:PACeqAGN}).
This is established by the following lemma.

\begin{lemma}[\cite{DBLP:journals/dam/Ben-DavidL98}]
If every finite subset of $\hh$ has a sample compression scheme of size $k=k(m)$,
then $\hh$ has a sample compression scheme of size $k=k(m)$.
\end{lemma}
A version of this lemma was proven by~\cite{DBLP:journals/dam/Ben-DavidL98}
for the case of sample compression scheme of a fixed size. Below we adapt their proof to
sample compression schemes of variable size.
The idea of the proof is to represent the statement ``$\hh$ has a sample compression scheme of size $k$''
using predicate logic, and to use the compactness theorem.
\begin{proof}[Proof of Theorem~\ref{t24}]
For ease of presentation, we shall consider only sample compression schemes with no side information.
The case of sample compression schemes with side information can be handled similarly.

Consider the model $M=\langle \hh,\X,\Y;R;h,x,y \rangle_{h\in\hh,x\in\X,y\in\Y}$,
where $R=\{(h,x,y) : h\in\hh, h(x)=y \}$.
The language for $M$ has the predicate symbols $\bar \hh, \bar \X, \bar Y$, and
has a constants $\bar h, \bar x, \bar y$ for every $h\in\hh, x\in \X,y\in\Y$.
More over, for every $k\in\mathbb{N}$, there is an $2k+2$ predicate symbol
$\bar \rho_k$.
We think of $\bar \rho_k\left((x_1,y_1),\ldots,(x_k,y_k),(x,y)\right)$,
as expressing that the reconstructed function $h=\rho\left((x_1,y_1),\ldots(x_k,y_k)\right)$ satisfies $h(x)=y$.

We now express the statement "$\hh$ has a sample compression scheme of size $k(m)$" using predicate logic over this language.
The following sentence expresses that $\rho$ reconstructs hypotheses (that is functions from $\X$ to $\Y$).
\begin{align*}
\psi_n=
&\left(\forall x_1,\ldots,x_n\in \bar \X\right)\left(\forall y_1,\ldots,y_n\in \bar \Y\right)
\left(\forall x\in\bar X\right)\left(\exists y\in \bar Y\right)\\
&~\bar \rho_n\left((x_1,y_1),\ldots,(x_n,y_n),(x,y)\right),\\
&\left(\forall y'\in\bar Y\right)
\left[
\bar \rho_n\left(\left(x_1,y_1\right),\ldots,\left(x_n,y_n\right),\left(x,y'\right)\right)\rightarrow y=y'
\right].
\end{align*}

The following sentence expresses the existence of a sample compression scheme of size $k=k(m)$:
\begin{align*}
\tau_m=
&\left(\forall x_1,\ldots,x_m\in \bar \X\right)\left(\forall y_1,\ldots,y_m\in \bar \Y\right)\\
&\left(\exists u_1,\ldots,u_k\in \bar \X\right)\left(\exists v_1,\ldots,v_k\in \bar \Y\right)\\
&\bigwedge_{i=1}^{k}\bigvee_{j=1}^{m}{\left(u_i=x_j\land v_i=y_j\right)}, \tag{$\{(u_i,v_i)\}_1^k\subseteq\{x_j,y_j\}_1^m$}\\
&\left[\left(\exists h\in \bar \hh\right)\bigwedge_{i=1}^{m}\bar R(h,x_i,y_i)\right]\rightarrow
\left[\bigwedge_{i=1}^{m}\bar \rho_k\left(\left(u_1,v_1\right),\ldots,\left(u_k,v_k\right),\left(x_i,y_i\right)\right)\right].
\tag{If $((x_i,y_i))_{i=1}^{m}$ is realizable, then the reconstruction agrees with it}
\end{align*}
Let $T_\rho = \{\psi_n,\tau_m:m,n\in\mathbb{N}\}$. Note that $T_\rho$ expresses that there exists a sample compression scheme for $\hh$ of size $k(m)$ with $\rho$ as a reconstruction function ($\kappa$ is defined implicitly to pick the sub-sample for which the reconstruction is correct).
Let $T_\hh$ denote the set of all atomic sentences and negation of atomic sentences
that hold in $M$. Define $T=T_\hh\cup T_\rho$.

Note that it suffices to show that $T$ is satisfiable. ($T_\hh\subseteq T$ implies that
every model for $T$ must contain a copy of $\hh$, and $T_\rho\subseteq T$ implies the existence of a
compression scheme of size $k(m)$.)
To this end we use the compactness theorem from predicate logic.
Since every finite subset of $\hh$ has a compression scheme of size $k(m)$, it follows that every finite
subset of $T$ is satisfiable. Therefore, by the compactness theorem $T$ is satisfiable and therefore has a model. This finishes the proof.
\end{proof}

\section{Graph dimension}\label{app:c26}

In this section we prove Theorem~\ref{t36} that relates the uniform convergence rate to the graph dimension.

\subsection{The lower bound}
We begin with the lower bound. 
Let $\hh$ be an hypothesis class with uniform convergence rate $d=d^{UC}(\epsilon,\delta)$, and
Let $d=\dim_{G}\left(\hh\right)$. We want to show that 
$$ d^{UC}(\eps,\delta) = \Omega\left(\frac{d+\log(1/\delta)}{\eps^2}\right).$$
Pick $f:X\to Y$ such that
$d=VC(\hh_{f})$. It follows that there is some set $\{x_1,\ldots,x_d\}\subseteq \X$ 
that is shattered by $\hh_{f}$. Consider the uniform distribution
$\D$ on $\left\{ \left(x_i,f\left(x_i\right)\right)\right\}$.

Given a sample $S=S=(z_1,\ldots,z_m)\sim\D^m$,
let $\hat p\in \mathbb{R}^d$ denote its type:
$$\hat p(i) = \frac{|\{j : z_j = (x_i,f(x_i)) \}| }{m}.$$
That is, $\hat p$ is a probability distribution that describes the fraction of times each example was observed in $S$.

Since $\{x_1,\ldots,x_d\}$ is shattered by $\hh_{f}$, for every $A\subseteq [d]$ there is $h_A\in \hh$ that disagrees with $f$
on $\{x_i : i\in A\}$ and agrees with $f$ on $\{x_i : i\notin A\}$. Note that for such $h_A$ we have that
$L_S(h_A)=\hat p(A)$ and $L_{\D}(h_A)=\frac{|A|}{d}$. 
Recall that the statistical distance between two probability measure $\mu,\nu$ on $\{1,...,d\}$ is defined by 
$$SD(\mu,\nu) = \frac{1}{2}\sum_{i=1}^d{\lvert \mu(i) - \nu(i)\rvert}=\max_{A\subseteq [d]} |\mu(A)-\nu(A)|. $$
In particular, the statistical distance between $\hat p$ and $(\frac{1}{d},\ldots,\frac{1}{d})\in\mathbb{R}^d$ is equal to
 $$\max_{h\in\hh}|L_\D(h)-L_S(h)|.$$ Thus, the lower bound follows from the following Theorem, 
{which is a corollary of results from~\cite{BerKont}. For the sake of completeness,
we provide a self contained proof of it that was suggested to us by Noga Alon.}
\begin{theorem}\label{td1}
Let $\eps < 1/200,\delta < 1/4$.
Let $u$ be the uniform distribution on $\{1,\ldots,d\}$.
For a sample $S\sim u^m$, let $\hat p$ denote the empirical
distribution observed in the sample, and let $SD(\hat p, u)$ denote the
statistical distance between $\hat p$ and $u$.
If $m\in\mathbb{N}$ satisfies
$$\Pr_{S\sim u^m}\left(SD\left(\hat p, u\right)\leq \eps\right)\geq 1-\delta,$$
then $m\geq C\frac{d+\log(1/\delta) - C}{\eps^2}$ for some constant $C$.
\end{theorem}
\paragraph{Remark.}
Theorem~\ref{td1} implies the lower bound of Theorem~\ref{t36} for $\eps < 1/200,\delta < 1/4$.
The lower bound for all $\eps,\delta\in(0,1]$ follows by from the monotonicity
of $d^{UC}(\eps,\delta)$ in $1/\eps$ and $1/\delta$,  and from the continuity of $\frac{\log(1/\delta)}{\eps^2}$
in $(0,1]\times(0,1]$.
\begin{proof}[Proof of Theorem~\ref{td1}]
For convenience we assume that $d$ is even (a similar proof applies in general).
We prove the Theorem in two steps. We first prove that $m = \Omega\left(\frac{\log(1/\delta)}{\eps^2}\right)$ 
and then prove that $m = \Omega\left(\frac{d}{\eps^2}\right)$.
We will use the following basic lemma.
\begin{lemma}\label{ld2}
Let $\eps<\frac{1}{4}$, and $\delta\leq\frac{1}{32}$.
Let $U$ denote the uniform distribution on $\{0,1\}$, and let $m\in\mathbb{N}$ such that
$$\Pr_{X\sim U^m}\left( \left\lvert\frac{\sum_{i=1}^m{X_i}}{m}-\frac{1}{2}\right\rvert\leq \eps\right)\geq 1-\delta.$$
Then $m\geq \frac{1}{24}\frac{\log(1/\delta)-5}{\eps^2}$.
\end{lemma}
We first finish the proof of Theorem~\ref{td1} and later prove Lemma~\ref{ld2}.
\paragraph{Proving that $m = \Omega(\frac{\log(1/\delta)}{\eps^2})$.}
Recall that
$$SD(\hat p, u) = \max_{A\subseteq [d]}\left\lvert \hat p(A) - \frac{\lvert A\rvert}{d} \right\rvert.$$
Pick $A=\{1,\ldots,d/2\}$. For $i=1,\ldots m$ let $X_i$ be the indicator of the event that the $i$'th
sample belongs to $A$. Thus, the $X_i$'s are independent uniform zero/one random variables and 
$$\hat p(A) = \frac{\sum_{i=1}^m{X_i}}{m}.$$
Therefore, if $m$ satisfies
$$\Pr_{S\sim u^m}\left(SD\left(\hat p, u\right)\leq \eps\right)\geq 1-\delta,$$ 
then in particular
$$\Pr_{S\sim u^m}\left( \left\lvert\frac{\sum_{i=1}^m{X_i}}{m}-\frac{1}{2}\right\rvert\leq \eps\right)\geq 1-\delta,$$
which, by Lemma~\ref{ld2}, implies that $m\geq \frac{1}{24}\frac{\log(1/\delta)-5}{\eps^2}$.

\paragraph{Proving that $m = \Omega(\frac{d}{\eps^2})$.}
Recall that 
$$SD(\hat p, u) = \frac{1}{2}\sum_{i=1}^d{\lvert\hat p(i) - 1/d\rvert}.$$
Partition $[d]$ to $\frac{d}{2}$ pairs
$$\{1,2\},\{3,4\},\ldots$$
For each pair $\{i,i+1\}$ let $W_i$ be the indicator of the event
$$\lvert\hat p(i) - 1/d\rvert \geq 100\eps/d~~ \mbox{ or } ~~\lvert\hat p(i+1) - 1/d\rvert \geq 100\eps/d.$$
We prove that: 
\begin{equation}\label{eq3}
\mbox{If $m<\frac{1}{24}\cdot\frac{d\log(1/50) - 5}{3\cdot(100\eps^2)}$ then }\Pr\left(\sum_i W_i > d/50\right)\geq{1}/{4}.
\end{equation}
This will finish the proof since it implies that for $m<\frac{1}{24}\cdot\frac{d\log(1/50) - 5}{3\cdot(100\eps^2)}$
$$SD(\hat p , u)= \frac{1}{2}\sum_{i=1}^d{\lvert\hat p(i) - 1/d\rvert} >\frac{1}{2}\cdot\frac{d}{50}\cdot\frac{100\eps}{d}=\eps$$ 
with probability at least $1/4$.

In order to derive~(\ref{eq3}), imagine that $\hat p$ is drawn according to the following
two-step sampling: 
\begin{itemize}
\item sample uniformly $m$ independent pairs from $\{\{1,2\},\{3,4\},\ldots\{d-1,d\}\}$, and
\item randomly replace each sampled pair $\{i,i+1\}$ by either $i$ or by $i+1$, each with probability~$1/2$.
\end{itemize} 
The advantage of this approach is that conditioned on the outcome of step (i), the $W_i$'s
are independent.

Fix some $i$ and consider $\hat p(\{i,i+1\})$ | the fraction of samples
that were equal to $\{i,i+1\}$ in step (i).
We show that for every possible conditioning on $\hat p(\{i,i+1\})$,
the probability that $W_i=1$ is at least $1/50$.
Consider two cases: if $\lvert\hat p(\{i,i+1\}) - {2}/{d}\rvert \geq 200\eps/d$
then $W_i=1$ for any possible replacement of the samples $\{i,i+1\}$ by $i$ or $i+1$ in step (ii). 
Thus, in this case $W_i=1$ with probability $1$.
Otherwise, if $\lvert\hat p(\{i,i+1\}) - {2}/{d}\rvert \leq 200\eps/d$,
then the number of samples equals to $\{i,i+1\}$ satisfies
$$\hat p\{i,i+1\}\cdot m \leq \frac{2+200\eps}{d}m\leq \frac{3}{d}m < \frac{1}{24}\cdot\frac{\log(1/50)-5}{(100\eps)^2}.$$
Therefore, by Lemma~\ref{ld2}, with probability at least $1/50$ in step (ii):
\begin{align*}
&\left\lvert\frac{\hat p(i)}{\hat p(\{i,i+1\})} - 1/2  \right\rvert > 100\eps \tag{by Lemma~\ref{ld2}}\\
&\implies W_i=1 \tag{since $\eps < \frac{1}{200}$, and $\hat p(\{i,i+1\})\geq \frac{2-200\eps}{d}$}\\
\end{align*}
Thus, conditioned on any possible outcome of step (i), the $W_i$'s are $d/2$ independent zero/one
random variables, and each of them is equal to $1$ with probability at least $1/50$.
Therefore, their sum is at least $\frac{d}{50}$ with probability at least $1/4$. 
This finishes the proof of Theorem~\ref{td1}.
\end{proof}
\begin{proof}[Proof of Lemma~\ref{ld2}]
This proof uses Shannon's entropy function and two of basic properties of it:
\begin{itemize}
\item if $p\in\mathbb{R}^n$ is a probability vector then
\begin{equation}\label{eq:h1}
H(p) \leq \log(n).
\end{equation}
\item For every $\eps\in [-1,1]$,
$$H(\frac{1}{2}-\eps,\frac{1}{2}+\eps) = 1 - \log(e)\sum_{k=1}^{\infty}\frac{(2\eps)^{2k}}{2k(2k-1)}.$$
In particular, if $\eps\leq\frac{1}{4}$ then
\begin{equation}\label{eq:h2}
H(\frac{1-\eps}{2},\frac{1+\eps}{2}) \geq 1 - 4\eps^2.
\end{equation}
\end{itemize}

The assumption on $m$ is equivalent to
$\left\lvert\left\{v\in \{0,1\}^m : \lvert \frac{1}{m}\sum_i v_i -\frac{1}{2}\rvert > \eps \right\}\right\rvert \leq \delta 2^m$, which
by symmetry is equivalent to
$$\left\lvert\left\{v\in \{0,1\}^m :  \sum_i v_i  < (\frac{1}{2} - \eps)m \right\}\right\rvert \leq \frac{\delta}{2} 2^m.$$
Let $B(m,r)$ denote the hamming ball $\left\{v\in \{0,1\}^m :  \sum_i v_i  < r\cdot m \right\}$. 
Thus, we wish to show that if $|B(m,\frac{1}{2}-\eps)| < \frac{\delta}{2}2^m$, then $m\geq \frac{1}{24}\frac{\log(1/100\delta)}{\eps^2}$.

Let $V$ be a random vector in $\{0,1\}^m$ such that for every $i$: $V_i=1$ with probability $\frac{1}{2}-\eps$ independently.
We bound $H(V)$ in two different ways.
First:
\begin{align*}
H(V) &= \sum_{i=1}^{m}H(V_i) \tag{by the chain rule and independence of the $V_i$'s}\\
        &= mH(\frac{1}{2}-\eps,\frac{1}{2}+\eps)\\
        &\geq m(1-4\eps^2) \tag{by Equation~\ref{eq:h2}}
\end{align*}
Second, set $B=B(m,\frac{1}{2}-\eps)$ and let ${\bf 1}_B$ denote the indicator of the event ``$V\in B$''.
\begin{align*}
H(V) &= H({\bf 1}_B, V) \\
        &= H({\bf 1}_B) + H(V \vert {\bf 1}_B) \tag{by the chain rule}\\
        &= H({\bf 1}_B) + \Pr(B)H(V \vert B) + \Pr(\bar B)H(V \vert \bar B) \\
        &\leq 1+\Pr(B)\log(|B|) + \Pr(\bar B)\log(\lvert\bar B\rvert) \tag{by Equation~\ref{eq:h1}}\\
        &\leq 1+\Pr(B)\log(\frac{\delta}{2} 2^m) + \Pr(\bar B)\log(2^m) \tag{$\lvert B\rvert \leq \frac{\delta}{2} 2^m,\lvert \bar B\rvert \leq 2^m$ }\\
        &= m+1 - \Pr(B)\log\frac{2}{\delta}.
\end{align*}
Therefore:
$$ m(1-4\eps^2) \leq H(V) \leq m+1-\Pr(B)\log\frac{2}{\delta},$$
which implies that $m\geq \frac{\Pr(B)\log\frac{2}{\delta}-1}{4\eps^2}$. The proof is finished by plugging $\Pr(B)\geq \frac{1}{6}$. That $\Pr(B)\geq \frac{1}{6}$ follows from the fact that either $\lceil (\frac{1}{2}-\eps)m\rceil$ or $\lfloor (\frac{1}{2}-\eps)m\rfloor$ is a median for the distribution $Bin(m,\frac{1}{2}-\eps)$~\cite{KB80}; therefore, since $B=B(m,\frac{1}{2}-\eps)$, either $\Pr(B)\geq\frac{1}{2}$ and we are done, or $\Pr(B) < \frac{1}{2}$ and $\Pr(B(m,\frac{1}{2}-\eps + \frac{1}{m}))\geq\frac{1}{2}$. In either case, a simple calculation shows that $\Pr(B)>\frac{1}{6}$.


\end{proof}

\subsection{The upper bound}
We now prove the upper bound, for which we need the next lemma.

\begin{lemma}\label{lem:aux}
Let $(\Omega,\mathcal{F},\mu)$ and
$(\Omega',\mathcal{F}',\mu')$ be countable\footnote{A similar
statement holds in general.} probability spaces.
Let
$$F_1,F_2,F_3,\ldots  \in \mathcal{F} , \
F'_1,F'_2,F'_3,\ldots \in \mathcal{F}'$$ be so that
$\mu'(F'_{i})\geq 1/2$ for all $i$.
Then
$$\left[\mu \times \mu' \right]\left(\bigcup_i {F_i\times F'_i} \right)\geq
\frac{1}{2}\mu \left( \bigcup_i {F_i} \right),$$
where $\mu\times \mu'$ is the product measure.
\end{lemma}

We finish the proof of the upper bound and later prove Lemma~\ref{lem:aux}.

For $h \in \hh$, define the event
$$F_h =  \{ Z : |L_Z(h) - L_\D(h)| > \eps  \},$$
and let
$F = \bigcup_{h \in \hh} F_h$.
Our goal is thus to upper bound $\Pr(F)$.
For that, we also define the independent event
$$F'_h = \{ Z' : |L_{Z'}(h)-L_\D(h)| < \eps/2 \}$$
where $Z' = (z'_1,\ldots,z'_m)$ are another $m$ independent samples from $\D$,
chosen independently of $Z$.
We first claim that $\Pr(F'_h) \geq 1/2$ for all $h \in H$.
This follows from Chernoff's bound
(but even Chebyshev's inequality suffices).
Now,
\begin{align*}
\Pr(F)
&\leq 2\Pr \left( \bigcup_{h \in H} F_h \times F'_h \right). \tag{Lemma~\ref{lem:aux}}
\end{align*}
Let $S = Z \cup Z'$, where the union is as multisets.
Imagine that $Z,Z'$ are sampled
as follows. First sample $S$, namely $2m$ independent samples $z_1,\ldots,z_{2m}$ from $\D$. Next sample $Z,Z'$ by uniformly drawing a partition of $S$ to two parts of size $m$.
Thus,
\begin{align*}
2\Pr \left( \bigcup_{h \in H} F_h \times F'_h \right) &=  2 \Ex_{S} \big[
\Ex \big[ 1_{  \exists h \in \hh :
|L_{Z}(h)-L_\D(h)| > \eps , |L_{Z'}(h)-L_\D(h)| < \eps/2 } \big| S \big] \big] \\
& \leq  2 \Ex_{S} \big[
\Ex \big[ 1_{ \exists h \in \hh :
|L_{Z}(h)-L_{Z'}(h)| > \eps/2 } \big| S \big] \big] = \ldots
\end{align*}
For fixed $S$,
let $\F = \F_S$ be the class of functions $f: [2m] \to\{0,1\}$ of
the form
\begin{align*}
f(i) = f_h(i)  = 1_{h(x_i) \neq y_i},
\end{align*}
where $z_i = (x_i,y_i)$, for some $h \in \hh$.

We claim that $\F_S$ has VC dimension at most the
graph dimension of $\hh$. If all the $x_i$ are distinct, then this claim
follows from the definition of graph dimension. In general, if $C\subseteq [2m]$ is
shattered by $\F_S$, and $i,j\in C$ are distinct, then we must have $x_i\neq x_j$,
which implies that $|C|\leq \dim_G(\hh)$.
Indeed, if $x_i=x_j$
then for each $h \in \H$,
at least one of the four patterns on $\{i,j\}$
is missing.
%

Since the graph dimension of $\hh$ is $d$,
by Sauer's lemma, the size of $\F_S$ is at most $(2me/d)^d$.
In addition, for $h \in \hh$,
the values of $L_Z(h)$ and $L_{Z'}(h)$ depend only
on $f_h$ and on how $Z,Z'$ partition $S$:
$$L_Z(h) = \frac{1}{m}\sum_{j \in J_Z} f_h(j):= \sigma_Z(f_h).$$
Thus, we can continue
\begin{align*}
\ldots &=  2 \Ex_{S} \big[
\Ex \big[ 1_{ \exists f \in \F_S :
|\sigma_Z(f) -\sigma_{Z'}(f)| > \eps/2 } \big| S \big] \big] \\
& \leq  2 \Ex_{S} \left[ \sum_{f \in \F_S}
\Pr \big[
|\sigma_Z(f) -\sigma_{Z'}(f)| > \eps/2  \big| S \big] \right] .
\end{align*}

It hence suffices to show that for every fixed $S$:
\begin{equation}\label{eq:qed}
\sum_{f \in \F_S}
\Pr \big[
|\sigma_Z(f) -\sigma_{Z'}(f)| > \eps/2  \big| S \big] \leq 2(2me/d)^d \exp\left(-m\eps^2/8\right).
\end{equation}

Fix $S$ and $f \in \F_S$.
Let $\sigma = \sigma_Z(f)$ and
$\sigma' = \sigma_{Z'}(f)$.
Now, the event that $\lvert\sigma-\sigma'\rvert\geq\eps/2$ can be described as follows.
$\sigma$ is the average of a random sample of size $m$ from $(f(1),\ldots,f(2m))$ without replacements,
and $\sigma'$ is the sum of the complementing sample. Note that $\sigma$ and $\sigma'$ are identically distributed.
By Hoeffding's inequality in the setting without replacement {(see, e.g.~\cite{opac-b1134861})} it holds that for $\mu = \mathbb{E}(\sigma)$
we have that
$$\Pr\left[\lvert \sigma - \mu \rvert \geq \eps/4 \right] \leq \exp\left(-2m(\eps/4)^2\right) = \exp\left(-m\eps^2/8\right).$$
By the union bound, this means that $\lvert\sigma-\sigma'\rvert\geq\eps/2$ with probability at most $2\cdot\exp\left(-m\eps^2/8\right)$.
Equation~\ref{eq:qed} follows since the size of $\F_S$ is at most $(2me/d)^d$.

\begin{proof}[Proof of Lemma~\ref{lem:aux}]
Let $F = \bigcup_i F_i$.
For every $\omega \in F$,
let $F'(\omega)= \bigcup_{i: \omega \in F_i} F'_i$.
As there exists $i$ such that $\omega\in F_i$ it holds that $F'_i\subseteq F'(\omega)$ and hence $\mu'(F'(\omega)) \geq 1/2$.
Thus,
\begin{align*}
\left[\mu \times \mu'\right] \left(\bigcup_i {F_i\times F'_i} \right)
= \sum_{\omega \in F} \mu(\{\omega\}) \cdot \mu'(F'(\omega))
\geq \sum_{\omega \in F} \mu(\{\omega\})/2 = \mu(F)/2.
\end{align*}
\end{proof}

\end{document}